\documentclass[11pt]{article}
\usepackage[font=libertinus,citestyle=authoryear]{kurbanlab}
\usepackage{needspace,colortbl}
\DeclareAffiliation{tamu}{Department of Electrical and Computer Engineering, Texas A\&M University, College Station, Texas, USA}
\DeclareAffiliation{tamuq}{Department of Electrical and Computer Engineering, Texas A\&M University at Qatar, Doha, Qatar}
\DeclareAffiliation{ankara}{Department of Prosthetics and Orthotics, Ankara University, Ankara, Turkey}
\DeclareAffiliation{hbku}{College of Science and Engineering, Hamad Bin Khalifa University, Doha, Qatar}

\newcommand{\Dset}{\mathcal{D}}
\newcommand{\Fset}{\mathcal{F}}
\newcommand{\Rset}{\mathcal{R}}

\newcommand{\Alg}{\mathcal{A}}
\newcommand{\Floor}{\Phi}
\newcommand{\Cert}{C}
\newcommand{\Probe}{\widehat{C}}
\newcommand{\Ex}{\mathbb{E}}

\definecolor{oiBlue}{RGB}{0,114,178}
\definecolor{oiOrange}{RGB}{230,159,0}
\definecolor{oiGreen}{RGB}{0,158,115}
\definecolor{oiRed}{RGB}{213,94,0}
\definecolor{oiGrey}{RGB}{110,110,110}

\title{Bounding Retraining Equivalence and the Deletion Floor in Materials Machine Unlearning}
\RunningTitle{Retraining Equivalence and the Deletion Floor}
\Author[orcid=0000-0002-1458-302X]{Can Polat}{tamu}
\Author[corresponding=kurbanm@ankara.edu.tr,orcid=0000-0002-7263-0234]{Mustafa Kurban}{tamuq,ankara}
\Author[orcid=0000-0001-9069-770X]{Erchin Serpedin}{tamu}
\Author[corresponding=hkurban@hbku.edu.qa,orcid=0000-0003-3142-2866]{Hasan Kurban}{hbku}
\hypersetup{pdftitle={Bounding Retraining Equivalence and the Deletion Floor in Materials Machine Unlearning},pdfsubject={Materials machine unlearning preprint},pdfkeywords={machine unlearning, materials, retraining equivalence, deletion floor, data redundancy}}
\begin{document}
\maketitle

\begin{abstract}
In materials machine learning, closely related retained structures can sustain accurate property predictions even after removing a specific record, rendering post-deletion prediction error an ambiguous metric for machine unlearning. To resolve this ambiguity, we define the \emph{deletion floor} as the expected target loss under a specified retraining procedure at the deleted request. Standard indistinguishability constraints yield a sharp interval bounding an update's target loss around this baseline reference. Theoretically, a conditional neighbor bound links a low deletion floor directly to retained fit, prediction regularity, and local label agreement, while an exact ridge identity isolates residual fit from the prediction change induced by record deletion. Empirically, controlled redundancy sweeps show an $\approx8\times$ drop in median normalized retraining loss when one retained relative remains after deletion. Across two distinct fitting regimes in a paired Materials Project study, the lower-floor regime also exhibits a larger prediction change on more than $50\%$ of the shared requests. Systematic comparisons against approximate updates and the original model decouple deliberate target suppression from preserved overall model utility. Consequently, request-level unlearning evaluations should report reference loss, prediction change, and retained utility together, interpreting post-deletion accuracy against what retraining itself leaves behind.
\end{abstract}

\section{Introduction}
\label{sec:introduction}

Materials-property datasets map chemical compositions and crystal structures directly to thermodynamic targets \citep{ward2018matminer,jain2013materials}. Because ionic substitutions routinely preserve underlying crystal lattices across varying chemical compositions, physically interpretable relationships naturally form among distinct database entries \citep{hautier2011substitution}. Recent findings in data duplication reveal that retained related examples can sustain original model behavior even after exact retraining \citep{ye2025duplication}. Consequently, removing a specific record's training contribution differs fundamentally from eliminating the broader physical behavior represented by that record \citep{triantafillou2026untraining}.

Standard certified removal frameworks evaluate updated models by comparing their output distributions against references trained without the deleted records \citep{guo2020certified,sekhari2021remember}. In practice, however, widespread unlearning membership attacks frequently overestimate privacy guarantees \citep{hayes2024inexact}, while simple accuracy-gap measurements often conflict with distribution-based forgetting scores \citep{triantafillou2024competition}. Resolving what accuracy remains after deletion requires examining individual requests against a clear reference. Without establishing what exact retraining itself would predict for a removed target, continued post-deletion accuracy remains an ambiguous signal of unlearning performance.

To resolve this ambiguity, we study individual deletion requests through the \emph{deletion floor} defined as the expected target loss under a specified retraining procedure. When a model accurately fits retained data and exhibits smooth prediction behavior across neighboring inputs, structural and property redundancies preserve target predictability. This mechanism decouples retained support, reference matching, and deletion sensitivity. Mathematically, a conditional neighbor bound directly connects retained structural support to the deletion floor, while retraining equivalence restricts how far an update's loss can deviate from this reference. Furthermore, an exact ridge identity isolates residual model fit from the true prediction change induced by record deletion (Figure~\ref{fig:teaser}).

\begin{figure}[t]
\centering
\includegraphics[width=\linewidth]{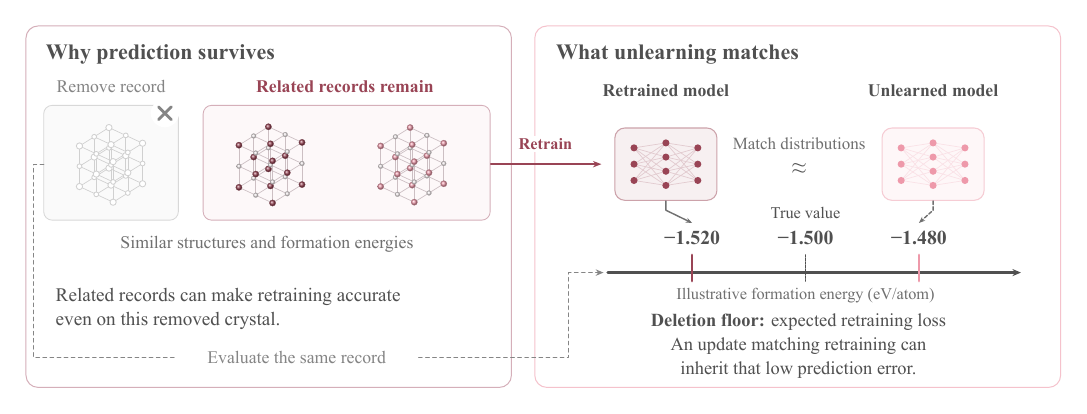}
\caption{\textbf{Related materials can keep a removed record predictable.} With good retained fit and smooth predictions, remaining relatives can keep retraining loss low; an update matching retraining can preserve that accuracy. The dashed path brings the removed record to the formation-energy comparison, where model arrows identify illustrative predictions beside its true value. The deletion floor is expected retraining loss at that record. Distributional certification requires model-law evidence beyond agreement of individual predictions.}
\label{fig:teaser}
\end{figure}

We validate these theoretical distinctions through a controlled redundancy sweep alongside $44{,}344$ crystal structures from the Materials Project (MP) \citep{jain2013materials}. Exact fixed-feature ridge calculations evaluate the floor and prediction change on identical requests, while approximate updates and unedited baselines expose how target suppression trades off against overall model utility. This work delivers three key contributions:
\begin{itemize}\itemsep1pt
\item \textbf{Derives a request-level reference interval:} Formulates a sharp bounded-loss interval that precisely identifies the target losses compatible with approximate retraining equivalence.
\item \textbf{Establishes a theoretical mechanism for retained predictability:} Proves a conditional neighbor bound linking the deletion floor to retained support, and derives an exact ridge decomposition that separates residual fit from deletion-induced sensitivity.
\item \textbf{Proposes a robust evaluation protocol:} Demonstrates via synthetic redundancy sweeps and materials benchmarks why reference loss, prediction change, and preserved utility measure distinct properties and must be reported independently.
\end{itemize}

The remainder of this paper is organized as follows: Section~\ref{sec:related} positions our contributions within the unlearning literature, Section~\ref{sec:method} presents the mathematical characterization, and Section~\ref{sec:experiments} details experiments across synthetic sweeps, materials fitting regimes, and approximate update procedures. Section~\ref{sec:implications} outlines the proposed request-level reporting protocol. Appendix~\ref{app:proofs} provides mathematical proofs, Appendix~\ref{app:protocol} specifies experimental parameters, and Appendix~\ref{app:results} presents complete operating-point summaries and neural descriptor diagnostics.
\section{Related work}
\label{sec:related}

\paragraph{Retraining objectives and evaluation.}
Retraining establishes the baseline benchmark for eliminating a record's training contribution, with certified removal requiring that an updated model's distribution remain statistically indistinguishable from one trained entirely without that record \citep{guo2020certified,sekhari2021remember}. Because exact retraining is computationally expensive, practical approximate algorithms frequently trade theoretical guarantees for efficiency; SCRUB, for instance, balances forgetting objectives against retained utility without providing a distributional certificate \citep{kurmanji2023unbounded}. Consequently, robust evaluation hinges on clarifying both the target objective and the specific properties measured by an evaluation metric. Existing membership inference attacks frequently overestimate privacy protections \citep{hayes2024inexact}, while simple accuracy-gap metrics often diverge from distribution-based forgetting scores \citep{triantafillou2024competition}. Evaluating post-deletion performance thus requires isolating the exact baseline error left by retraining at individual deletion requests.

\paragraph{Duplication, memorization, and retained support.}
Data redundancy presents a distinct challenge for unlearning algorithms. Empirical studies by \citet{ye2025duplication} demonstrate that adversarially duplicated or near-duplicated training instances can preserve original model behavior even after exact retraining. Similarly, \citet{rum2024hard} connect the difficulty of approximate unlearning to data memorization and retain--forget entanglement, showing that low memorization leaves a model's prediction largely insensitive to the inclusion or exclusion of a specific training point.

\paragraph{Broader behavioral removal.}
Removing a specific record's training influence represents a distinct objective from eliminating the overarching behavioral concept or domain knowledge associated with that record \citep{triantafillou2026untraining}. Correlated-knowledge unlearning evaluates how inference persists through retained structural relationships \citep{wu2024deepunlearning}, whereas retain--forget entanglement focuses on unwanted collateral damage inflicted on related retained samples during forgetting \citep{cheng2026entanglement}. Additionally, distributional unlearning selects data subsets to excise entire target distributions while leaving complementary distributions intact \citep{allouah2026distributional}. Together, these formulations highlight the need to distinguish between broader distribution removal and sample-level reference agreement.

\paragraph{Positioning of this work.}
This work investigates the baseline target loss left by exact retraining at individual deletion requests. By defining the \emph{deletion floor}, we complement existing update algorithms and evaluation metrics by identifying the exact target behavior constrained by reference matching. Building on prior observations of persistence under data duplication and low memorization, we formalize retained support as a theoretical mechanism, connecting it directly to retraining equivalence, local neighbor fit, and ridge leverage. Our bounded-loss transfer theorem delineates the precise loss intervals permitted under indistinguishability, while our exact ridge decomposition decouples retraining accuracy from deletion-induced prediction changes. Paired materials experiments confirm that lower reference loss can coexist with larger prediction sensitivity. Ultimately, these findings establish a reference-based evaluation protocol for regression unlearning, while leaving broader knowledge and distribution removal as distinct objectives.
\section{Bounding loss around the retraining reference}
\label{sec:method}

\subsection{The deletion floor and bounded-loss transfer}
We first define the retraining reference following data deletion. Let $\Dset=((x_i,y_i))_{i=1}^n$ denote an indexed dataset, where $\Fset$ specifies requested deletion records and $\Rset$ contains remaining indices. Explicit indexing keeps identical record copies distinct. Applying a training procedure $\Alg_{\mathrm{tr}}$ to $\Rset$ induces a model distribution law $Q$ over training randomness. At a target deletion request $(x_f,y_f)$, let $h(\theta)=\ell(f_\theta(x_f),y_f)$ be a nonnegative measurable loss with finite expectation under $Q$.

\begin{definition}[Deletion floor]
\label{def:floor}
The deletion floor is the expected target loss under the specified retraining procedure,
\[
\Floor(x_f;\Fset,\Dset)=\Ex_{\theta\sim Q}[h(\theta)].
\]
\end{definition}
Definition~\ref{def:floor} quantifies the expected residual error left at the deleted target by retraining. Under squared loss, a low floor indicates that exact retraining inherently predicts the deleted target accurately. Because this benchmark reflects the specific retraining procedure $\Alg_{\mathrm{tr}}$, alternative estimators may yield losses above or below this reference.

To evaluate what an update algorithm constrained by retraining equivalence can alter, let $P$ represent its output law on the same measurable model space. Symmetric $(\epsilon,\delta)$ retraining equivalence requires $P(S)\le e^\epsilon Q(S)+\delta$ and $Q(S)\le e^\epsilon P(S)+\delta$ across all measurable sets $S$ \citep{guo2020certified}. By translating standard hypothesis-testing constraints to target loss observables \citep{kairouz2015composition}, we establish sharp bounds on the update's expected loss relative to the floor.

\begin{theorem}[Sharp bounded-loss transfer]
\label{thm:floor}
Let $B>0$, $\epsilon\ge0$, $\delta\in[0,1]$, and let $h$ be a measurable observable taking values in $[0,B]$. If $P,Q$ satisfy symmetric $(\epsilon,\delta)$ retraining equivalence and $\Floor=\Ex_Q[h]$, then
\begin{align*}
\max\!\left\{0,e^{-\epsilon}(\Floor-\delta B),B-e^\epsilon(B-\Floor)-\delta B\right\}
&\le\Ex_P[h]\\
&\le\min\!\left\{B,e^\epsilon\Floor+\delta B,B-e^{-\epsilon}(B-\Floor-\delta B)\right\}.
\end{align*}
For each fixed $(\epsilon,\delta,B,\Floor)$ with $0\le\Floor\le B$, both endpoints are attainable over bounded observables and admissible pairs, converging to $\Floor$ as $(\epsilon,\delta)\to(0,0)$.
\end{theorem}

Theorem~\ref{thm:floor} bounds the expected target loss deviation permitted under approximate equivalence. Applying the event inequalities to $h$ and $B-h$ yields the stated interval. Its endpoints converge to $\Floor$ as $(\epsilon,\delta)\to(0,0)$, while the two-point construction in Appendix~\ref{app:transfer} establishes sharpness.

\paragraph{Applying the guarantee to squared error.}
Because standard squared error is unbounded, we apply Theorem~\ref{thm:floor} by defining a clipped loss observable $h_B(\theta)=\min\{(f_\theta(x_f)-y_f)^2,B\}$ for a fixed threshold $B>0$. Measurable predictions yield integrable clipped losses $0\le h_B\le B$ under both $P$ and $Q$, anchored by the clipped reference expectation $\Floor_B=\Ex_Q[h_B]$. While theoretical bounds govern this clipped observable, our empirical analyses report unclipped raw squared loss to describe full retraining behavior. Local target agreement reflects retraining equivalence, but complete distributional equivalence requires model-law validation beyond scalar agreement at a single point.

\paragraph{A target-error threshold.}
When an unlearning objective requires increasing target error to $\Ex_P[h]\ge\tau$ (where $0<\Floor<\tau<B$) under pure $(\epsilon,0)$ equivalence, both upper bounds in Theorem~\ref{thm:floor} must simultaneously reach $\tau$. Solving for $\epsilon$ establishes the necessary equivalence budget:
\begin{equation}
\epsilon\ge\max\!\left\{\log\frac{\tau}{\Floor},\ \log\frac{B-\Floor}{B-\tau}\right\}.
\label{eq:target-threshold}
\end{equation}
The first term quantifies the required loss increase from $\Floor$ to $\tau$, while the second reflects the corresponding reduction in the complement expectation $B-h$. For example, elevating expected unit-bounded loss from $\Floor=0.1$ to $\tau=0.5$ requires $\epsilon\ge\log 5$. Equation~\eqref{eq:target-threshold} provides a necessary compatibility condition between target-suppression goals and formal equivalence guarantees; relaxed guarantees with $\delta>0$ introduce additive slack across the full interval.

\subsection{How retained support can keep the floor low}
To understand why retraining often leaves deleted targets predictable, consider a retained neighbor $(x_r,y_r)$ near the deleted input $x_f$ under distance metric $d$. If a retrained model accurately fits $x_r$, varies smoothly between inputs, and target labels align, the retained example constrains prediction error on the removed target.

\begin{assumption}[Uniform retained-neighbor control]
\label{as:lip}
For a fixed retained neighbor $(x_r,y_r)$, assume $f_\theta(x_f)$ and $f_\theta(x_r)$ are measurable in $\theta$. For $Q$-almost every $\theta$, suppose $|f_\theta(x_f)-f_\theta(x_r)|\le\bar Ld(x_f,x_r)$ and $|f_\theta(x_r)-y_r|\le\bar\varepsilon$ for fixed $\bar L,\bar\varepsilon\ge0$.
\end{assumption}
Bounding the input distance $d(x_f,x_r)\le\rho$ and label gap $|y_f-y_r|\le\eta$, target error decomposes into retained fit error, prediction variation, and label disparity:
\begin{equation}
\begin{aligned}
|f_\theta(x_f)-y_f|
&\le \underbrace{|f_\theta(x_r)-y_r|}_{\text{retained fit}}
 +\underbrace{|f_\theta(x_f)-f_\theta(x_r)|}_{\text{prediction variation}}
 +\underbrace{|y_r-y_f|}_{\text{label gap}}\\
&\le\bar\varepsilon+\bar L\rho+\eta\equiv\Cert(x_f).
\end{aligned}
\label{eq:neighbor-bound}
\end{equation}
Under Assumption~\ref{as:lip}, taking expectations bounds the absolute-loss floor by $\Cert$, while squaring first bounds the squared-loss floor by $\Cert^2$. Consequently, high retained accuracy, smooth local predictions, and small label gaps jointly guarantee a low deletion floor. For finite neighbor sets, the tightest valid bound holds almost surely across retrained models.

\paragraph{An empirical diagnostic.} 
We construct a request-level surrogate for the theoretical bound via
\[
\Probe(x_f)=\widehat\varepsilon+\widehat Ld(x_f,x_r)+|y_f-y_r|,
\]
where $x_r$ represents the nearest retained neighbor in standardized descriptor space. Here, $\widehat\varepsilon$ is estimated from the original model's $90$th-percentile absolute training residual, and $\widehat L$ is set to the $99.5$th percentile of sampled prediction difference quotients. Because its components are estimated from original-model quantiles rather than uniform bounds over retrained models, \textsc{FloorScore} is a ranking diagnostic, not a certified upper bound on the deletion floor. Neighbor selection is performed prior to evaluating deletion losses, avoiding direct use of the observed deletion losses in constructing the proxy.

\subsection{An exact ridge connection: fit, leverage and deletion}
While neighbor bounds establish sufficient conditions for low floors, ridge regression enables exact calculation of retraining residuals and decouples baseline model fit from deletion sensitivity. Consider a feature matrix $A\in\mathbb R^{n\times p}$ with rows $a_i^\top$, target vector $y\in\mathbb R^n$, and objective $n^{-1}\|Aw-y\|^2+\lambda\|w\|^2$ ($n\ge2$, $\lambda>0$) minimized by $w_\Dset$. Because row deletion alters mean-loss normalization from $n\lambda$ to $(n-1)\lambda$, we define an adjusted full-data estimator $\widetilde w$ using the retained-size penalty coefficient:
\begin{equation}
M=A^\top A+(n-1)\lambda I,\qquad
\widetilde w=M^{-1}A^\top y,\qquad
\widetilde h_i=a_i^\top M^{-1}a_i.
\label{eq:ridge-adjust}
\end{equation}
The adjusted fit $\widetilde w$ maintains all training rows while adopting the penalty scaling of the retained system. Applying the Sherman--Morrison identity \citep{hager1989updating} yields an exact closed-form expression for the retrained residual $r_{-i}$ and corresponding deletion floor $\Floor_i$:
\begin{equation}
r_{-i}:=a_i^\top w_{-i}-y_i
=\frac{a_i^\top\widetilde w-y_i}{1-\widetilde h_i},
\qquad \Floor_i=r_{-i}^2.
\label{eq:ridge-loo}
\end{equation}
The denominator is strictly positive because
$M-a_i a_i^\top=A_{-i}^\top A_{-i}+(n-1)\lambda I$
is positive definite, which implies $\widetilde h_i<1$. Equation~\eqref{eq:ridge-loo} reveals that retraining error is governed jointly by the adjusted full-data residual $a_i^\top\widetilde w-y_i$ and the deleted point's leverage $\widetilde h_i$. A small adjusted residual yields a low deletion floor provided leverage remains bounded away from one, whereas high full-data residual error carries over into large retraining loss.

Comparing predictions before and after deletion isolates model sensitivity to record removal. Defining full-model residual $r_i=a_i^\top w_\Dset-y_i$, deletion response is captured by target prediction change $|r_i-r_{-i}|$. Because squared losses discard residual signs, evaluating prediction change requires signed residual tracking. We decompose the total prediction change into penalty rescaling and row removal components:
\begin{equation}
a_i^\top(w_{-i}-w_\Dset)
=\underbrace{\lambda a_i^\top M^{-1}w_\Dset}_{\text{penalty rescaling}}
+\underbrace{\frac{\widetilde h_i\widetilde r_i}{1-\widetilde h_i}}_{\text{removing the row}},
\label{eq:ridge-change}
\end{equation}
where $\widetilde r_i=a_i^\top\widetilde w-y_i$. Equation~\eqref{eq:ridge-change} demonstrates that a highly accurate model fit can still exhibit strong prediction sensitivity upon record removal, motivating joint measurement of retraining floors and deletion changes.

\paragraph{Why the no-update reference matters.}
The unedited original model serves as a necessary control because baseline predictions may already align closely with retraining outcomes. To formalize this intuition, consider $n\ge2$ and differentiable, $\mu$-strongly convex objectives ($\mu>0$) sharing a regularizer. Assuming per-example gradient norms are bounded by $G$ and model predictions are $L_p$-Lipschitz in parameters ($G,L_p\ge0$), unconstrained minimizers satisfy the leave-one-out stability bound:
\begin{equation}
|f_{w_\Dset}(x_f)-f_{w_\Rset}(x_f)|\le\frac{2GL_p}{\mu n}.
\label{eq:stability}
\end{equation}
The $O(1/n)$ scaling in Equation~\eqref{eq:stability} confirms that small original-to-retrained prediction shifts occur naturally in stable models, highlighting why no-update controls are indispensable when evaluating update agreement against retraining benchmarks.
\section{Experiments and results}
\label{sec:experiments}

Empirical evaluations demonstrate how data redundancy and fitting capacity shape the retraining reference, establishing the benchmark against which approximate unlearning updates and unedited baselines are assessed across matched requests.

\subsection{Shared reference and measurement protocol}
Evaluating deletion behavior requires tracking both loss relative to the retraining reference and physical prediction shifts. For each single-example deletion request $i$, we record the squared target loss $e_i$, its corresponding retraining floor $\Floor_i$, and the normalized error ratio $q_i=e_i/\max(\Floor_i,10^{-12})$. Predictive utility is measured by test MSE relative to the original model. To isolate distinct prediction shifts, let $\widehat y_i^D$, $\widehat y_i^R$, and $\widehat y_i^U$ denote predictions from the original, retrained, and updated models, respectively. The \emph{deletion-induced prediction change} $|\widehat y_i^D-\widehat y_i^R|$ measures original model sensitivity to record removal, while the \emph{update-to-reference prediction distance} $|\widehat y_i^U-\widehat y_i^R|$ measures update deviation from retraining. The scalar loss gap $g_i=|e_i-\Floor_i|$ relates directly to the normalized ratio via
\begin{equation}
g_i=\left|\max(\Floor_i,10^{-12})q_i-\Floor_i\right|.
\label{eq:dependent-metrics}
\end{equation}
The safeguard denominator $10^{-12}$ ensures numeric stability for near-zero reference losses without altering observed ratios, as all measured ridge floors exceed $1.12\times10^{-7}$~(eV/atom)$^2$. Thus, Equation~\eqref{eq:dependent-metrics} unifies target loss comparisons while prediction distances capture directional shifts.

To isolate data-level effects, the materials ridge reference refits weights on retained rows while holding feature standardization, RFF kernel bandwidth, and random projections fixed \citep{rahimi2007random}. This conditional procedure, denoted $\Alg_{\mathrm{tr}}(\cdot;T_\Dset)$, preserves the request's contribution to learned preprocessing $T_\Dset$. Under deterministic ridge fitting, the squared residual equals the expected loss floor under fixed hyperparameters, with across-request medians reporting central tendencies. Score thresholds are calculated within each diagnostic sample, and request-level distributions accompany aggregate medians.

\subsection{Retraining residuals across controlled redundancy settings}
Retained family members substantially reduce retraining loss even after a target record is removed. We evaluate synthetic datasets constructed with $90$ prototype families and $160$ isolated points in a six-dimensional descriptor space, where labels generated from a smooth random-feature function are corrupted by additive noise. Across family sizes $k\in\{1,2,4,8,16\}$ and total dataset sizes $90k+160$, we assess $300$ independent trials, deleting one family member and one isolated control separately and normalizing floors by original test MSE.

Transitioning from $k=1$ (where deletion removes the sole family member) to $k=2$ (leaving one retained relative) produces a dramatic $7.934\times$ drop in median normalized retraining loss, from $0.326$ down to $0.041$. Additional family members yield diminishing loss reductions, reaching a median floor of $0.020$ ($\approx 2\%$ of test MSE) at $k=16$. In contrast, isolated controls retain higher median normalized retraining losses than family targets throughout the sweep, consistent with retained family support contributing to lower target loss.

\subsection{Materials floors depend on the fitting regime}
Deletion floors vary systematically with model capacity and regularization. We analyze $44{,}344$ crystal structures ($35{,}476$ train, $8{,}868$ test) from the MP using 134 structural and chemical descriptors targeting formation energy per atom. Two distinct ridge regimes evaluate exact retained residuals: a $2{,}048$-feature model ($\lambda=10^{-6}$, 200 diagnostic requests) and an $8{,}192$-feature model ($\lambda=10^{-8}$, 300 diagnostic requests), coupling feature count and regularization across separate diagnostic samples.

\begin{figure}[!t]
\centering
\includegraphics[width=\linewidth]{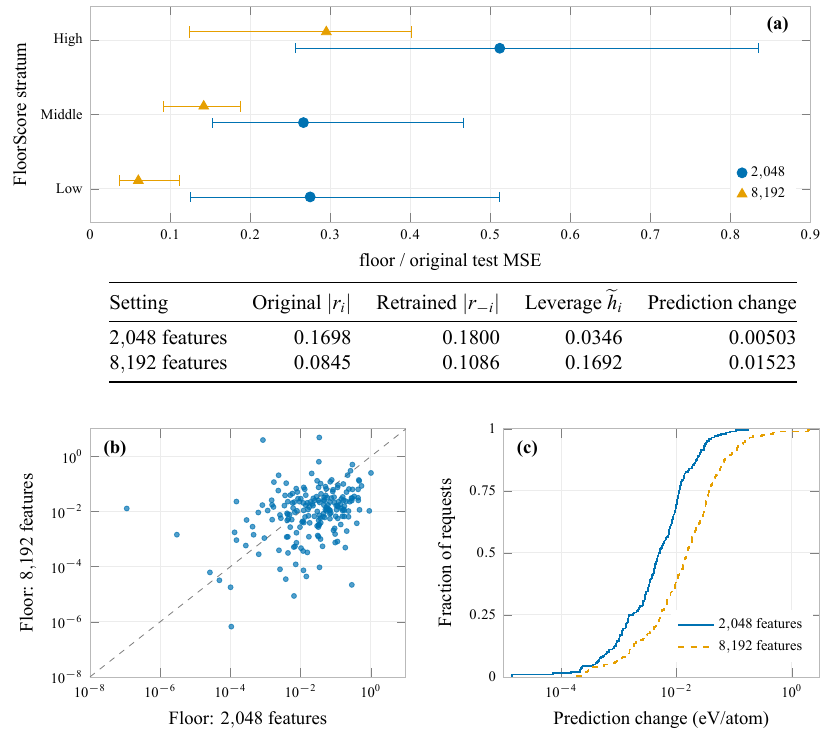}
\caption{\textbf{Lower retraining loss can accompany larger prediction changes after deletion.} (a) Median normalized floors across \textsc{FloorScore} strata for the $2{,}048$-feature ($\lambda=10^{-6}$) and $8{,}192$-feature ($\lambda=10^{-8}$) fits; error bars are $95\%$ request-bootstrap intervals. Their diagnostic samples contain $200$ and $300$ requests, respectively, and differ across fits. The embedded table reports medians on a separate paired cohort of $200$ shared requests: original and retrained absolute residuals and prediction changes are in eV/atom, while adjusted leverage is dimensionless. (b) Paired squared retraining losses in (eV/atom)$^2$; $132$ points fall below equality. (c) Empirical distributions of deletion-induced prediction changes; the $8{,}192$-feature setting has the larger change on $172$ requests. Feature count and regularization vary jointly.}
\label{fig:materials}
\end{figure}

The fitting regime strongly modulates the relationship between the local structural proxy \textsc{FloorScore} and actual retraining loss (Figure~\ref{fig:materials}(a)). In the $2{,}048$-feature setting, low-, middle-, and high-\textsc{FloorScore} requests exhibit median normalized floors of $0.275$, $0.267$, and $0.512$. The $8{,}192$-feature, lower-penalty regime lowers the corresponding medians to $0.0605$, $0.142$, and $0.295$—yielding a $4.88\times$ high-to-low contrast and rank correlation of $0.261$—while simultaneously reducing training MSE from $0.0754$ to $0.0166$ and test MSE from $0.0870$ to $0.0704$. These results show that \textsc{FloorScore} is a setting-dependent diagnostic whose empirical stratification varies with the fitting regime.

\subsection{Lower floors can coexist with larger deletion changes}
Evaluating identical deletion requests reveals that higher model accuracy does not imply lower deletion sensitivity. We compare both fitting regimes on a shared set of $200$ deletion requests using identical train/test splits.

The paired medians embedded in Figure~\ref{fig:materials} show that the $8{,}192$-feature setting has smaller original and retrained residuals but larger adjusted leverage and deletion-induced prediction change. Compared with the $2{,}048$-feature baseline, it achieves a lower retraining floor on $132$ of $200$ requests while exhibiting a larger prediction change on $172$ requests, dropping median squared retraining loss from $0.0324$ to $0.0118$~(eV/atom)$^2$ while increasing median prediction change from $0.00503$ to $0.0152$~eV/atom (Figure~\ref{fig:materials}(b)--(c)).

On $104$ individual requests, the $8{,}192$-feature, lower-penalty regime predicts removed targets more accurately after retraining despite moving predictions farther from the original fits. This decoupling highlights that retraining floors evaluate post-deletion accuracy, whereas prediction changes quantify target sensitivity.

\subsection{Target loss and retained utility at fixed operating points}
Approximate unlearning interventions reveal a trade-off between target suppression and preserved overall model utility. Using the shared cohort of $200$ requests in the $2{,}048$-feature setting ($\lambda=10^{-6}$), we compare approximate updates against exact retraining and the unedited no-update control (Table~\ref{tab:operating}).

Evaluated updates start from independent copies of the original model under fixed hyperparameter settings: NegGrad+ \citep{kurmanji2023unbounded} (30 updates combining target loss ascent with retain loss descent), SCRUB-style regression \citep{kurmanji2023unbounded} (40 alternating ascent/teacher-matching updates), Newton-style removal (1 damped Hessian step), retain fine-tuning (60 steps), random relabeling (40 steps), and gradient ascent (30 target ascent steps), all using learning rate $0.010$ and retained minibatches $\le 256$.

\begin{center}\small
\captionof{table}{\textbf{Increasing target error can also reduce predictive utility.} A target-loss ratio above one indicates greater error than retraining, and a test-MSE ratio above one indicates worse utility than the original model. Gradient ascent raises both ratios. Several other updates remain near the no-update control. Entries are medians over $200$ shared requests in full-corpus ridge with $2{,}048$ features and $\lambda=10^{-6}$ at the fixed update settings.}
\label{tab:operating}
\setlength{\tabcolsep}{7pt}
\begin{tabular}{@{}lrr@{}}\toprule
Procedure & Target loss / floor & Test MSE ratio\\\midrule
Retrain (reference) & $1.000$ & $1.000$\\
No update & $0.932$ & $1.000$\\
Retain fine-tuning & $0.930$ & $1.000$\\
Newton-style update & $0.934$ & $1.000$\\
Random relabeling & $0.928$ & $1.000$\\
SCRUB-style regression & $1.141$ & $1.001$\\
NegGrad+ & $1.645$ & $1.013$\\
Gradient ascent & $3.081$ & $1.083$\\
\bottomrule\end{tabular}\end{center}

The unedited control exhibits a median target-loss/floor ratio of $0.932$ and reference distance of $0.00503$~eV/atom. Retain fine-tuning, Newton-style updates, and random relabeling closely mirror the control's loss ratio and test utility. In contrast, gradient ascent forces target loss to $3.081\times$ the reference while increasing test MSE by $8.3\%$ ($1.083\times$ ratio). These results emphasize that target loss inflation alone does not signify successful unlearning unless evaluated alongside reference agreement and retained utility.

\section{Evaluation protocol and scope}
\label{sec:implications}

Unlearning evaluations must decouple post-deletion accuracy from record sensitivity: the deletion floor quantifies what a specified retraining reference predicts, whereas prediction change measures model sensitivity to record removal. Because higher post-retraining accuracy can coexist with larger prediction shifts, reporting both properties is necessary to contextualize unlearning performance.

\paragraph{A request-level reporting protocol.}
Evaluations should explicitly define the target deletion objective, training reference, and any refitted preprocessing steps. For each request, benchmarks should report original, retrained, and updated target losses, deletion-induced prediction changes, update-to-reference prediction distances (when available), and retained test utility, alongside an unedited original model as a no-update control. Retaining signed residuals enables exact recovery of prediction shifts under squared loss. Comparing predictions across original, retrained, and updated models isolates true update effects beyond baseline reference agreement, whereas proving formal retraining equivalence requires distributional certification beyond scalar point agreement.

\paragraph{Interpreting operating points.}
Evaluating approximate updates requires measuring target loss shifts against both reference floors and preserved utility. An update satisfying retraining equivalence naturally remains accurate when the reference deletion floor is low. Conversely, target-suppression interventions elevate loss above the floor, requiring separate utility accounting; for example, gradient ascent forces target loss higher while degrading overall test utility. Because squared loss discards residual signs and near-zero reference floors produce artificially inflated loss ratios, evaluations should report full request-level distributions and explicit loss normalizations alongside aggregate medians.

\paragraph{Scope and limitations.}
These findings reflect specific experimental and theoretical scopes. In controlled sweeps, global redundancy scales jointly with total dataset size, while in materials studies, feature dimension and regularization vary together under fixed descriptors, preprocessing, and model architectures. Crystal graph neural network diagnostics provide exploratory observations per request, and \textsc{FloorScore} serves as a setting-dependent proxy whose empirical rank correlation varies with model capacity. Additionally, leave-one-out stability bounds assume uniform gradient control. Broader behavioral removal, comparative hyperparameter tuning, distributional certification, and end-to-end refitting of learned representations remain distinct objectives requiring expanded frameworks. Within this scope, our protocol provides transparent accounting of retraining references and update behavior.

\section{Conclusion}
\label{sec:conclusion}

The deletion floor establishes an explicit baseline for evaluating regression unlearning by measuring the residual target error left by exact retraining. Retraining equivalence restricts updated models around this reference, retained structural support maintains low reference loss, and ridge leverage explains how accurate retraining can remain highly sensitive to record removal. Across synthetic redundancy sweeps and materials benchmarks, empirical measurements confirm that lower reference loss frequently coexists with larger prediction sensitivity. Consequently, continued post-deletion accuracy can only be interpreted relative to reference floors and target objectives. Standardizing unlearning evaluation requires reporting retraining predictions, record removal shifts, and retained model utility together.

\section*{Reproducibility statement}
The accompanying repository (\url{https://github.com/KurbanIntelligenceLab/materials-unlearning}) contains implementation details, data provenance, environment information, result records, and verification scripts. It also includes Lean 4 formalizations verifying bounded-loss transfer and sharpness, endpoint convergence, target-error thresholds, retained-neighbor bounds, ridge identities, and conditional leave-one-out stability.

\section*{Ethics statement}
This study uses synthetic data and public materials-property records, involving no human subjects or personal data. Demonstrating that retained structural support preserves target accuracy after deletion is a technical observation and should not be used to infer compliance with specific legal or regulatory data removal mandates. Applications seeking broader behavioral concept removal require evaluations tailored to that specific objective.

\section*{Use of AI statement}
Generative AI tools assisted with literature review, text polishing, code development, and data analysis. The submitting authors retain full responsibility for the final text, theoretical proofs, cited sources, code, and reported empirical findings.

\bibliography{delfloor}

\clearpage
\appendix
\section{Proofs of the main characterization}
\label{app:proofs}
\subsection{Proof of Theorem~\ref{thm:floor}: bounded-loss transfer}
\label{app:transfer}

The bounded-loss transfer theorem establishes how symmetric $(\epsilon,\delta)$ retraining equivalence restricts expected loss shifts relative to the retraining reference floor $\Floor$.

\begin{proof}[Proof of Theorem~\ref{thm:floor}]
Applying hypothesis-testing constraints on bounded observables \citep{kairouz2015composition}, let $a=e^\epsilon$. For any measurable loss $h\in[0,B]$, the layer-cake representation combined with distributional indistinguishability $P(S)\le aQ(S)+\delta$ yields
\[
\Ex_P[h]
=\int_0^B P(h>t)\,dt
\le a\int_0^B Q(h>t)\,dt+\delta B
=a\Floor+\delta B.
\]
Reversing the roles of $P$ and $Q$ gives the lower bound $\Ex_P[h]\ge a^{-1}(\Floor-\delta B)$. Applying these identical constraints to the complementary bounded observable $B-h$ establishes
\[
\Ex_P[h]\ge B-a(B-\Floor)-\delta B,
\qquad
\Ex_P[h]\le B-a^{-1}(B-\Floor-\delta B).
\]
Intersecting these four inequalities with the physical loss range $0\le\Ex_P[h]\le B$ completes the proof of the transfer interval.

To prove sharpness, normalize by $B$ and define $q=\Floor/B$. Consider a two-point sample space with observable $h=B\mathbf{1}_S$, setting $Q(S)=q$ and $P(S)=p$. Symmetric $(\epsilon,\delta)$-indistinguishability reduces to the four joint constraints
\[
p\le aq+\delta,\quad q\le ap+\delta,\quad
1-p\le a(1-q)+\delta,\quad 1-q\le a(1-p)+\delta,
\]
subject to $0\le p\le1$. The feasible interval for $p$ corresponds exactly to the normalized upper and lower bounds of Theorem~\ref{thm:floor}. Feasibility is guaranteed because $p=q$ satisfies all constraints for valid parameters. Endpoint attainability holds across all $0\le\Floor\le B$, including extreme points $q=0$ and $q=1$. As $(\epsilon,\delta)\to(0,0)$, the maximum defining the lower endpoint and the minimum defining the upper endpoint both converge to $\Floor$.
\end{proof}

\paragraph{Justification of Equation~\eqref{eq:target-threshold}.}
Under pure $(\epsilon,0)$ equivalence ($\delta=0$), requiring updated target error to reach $\Ex_P[h]\ge\tau$ (where $0<\Floor<\tau<B$) forces both upper bounds in Theorem~\ref{thm:floor} to be at least $\tau$, namely $e^\epsilon\Floor\ge\tau$ and $B-e^{-\epsilon}(B-\Floor)\ge\tau$. Taking logarithms yields the two necessary constraints in Equation~\eqref{eq:target-threshold}, defining the minimum equivalence budget $\epsilon$ compatible with the target error objective.

\subsection{Justification of the retained-neighbor bound}
\label{app:neighbor}

\begin{proof}[Justification of Equation~\eqref{eq:neighbor-bound}]
Under uniform retained-neighbor control (Assumption~\ref{as:lip}), target error at deletion request $x_f$ decomposes into retained neighbor fit, local prediction variation, and label disparity. For $Q$-almost every retrained predictor $\theta$, triangle inequality yields
\[
|f_\theta(x_f)-y_f|\le |f_\theta(x_f)-f_\theta(x_r)|+
|f_\theta(x_r)-y_r|+|y_r-y_f|\le\bar L\rho+\bar\varepsilon+\eta\equiv\Cert(x_f).
\]
Because $\Cert(x_f)$ is deterministic and nonnegative, taking expectations bounds absolute loss by $\Ex_Q[|f_\theta(x_f)-y_f|]\le\Cert(x_f)$, while squaring first yields squared loss bound $\Ex_Q[(f_\theta(x_f)-y_f)^2]\le\Cert(x_f)^2$. For clipped squared loss, expectation is bounded by $\min\{\Cert(x_f)^2,B\}$. For any finite collection of valid retained neighbors, intersecting their measure-one validity sets preserves a measure-one domain where the minimum bound remains valid.
\end{proof}

\subsection{Derivation of the ridge identities}
\label{app:ridge}

\begin{proof}[Derivation of Equations~\eqref{eq:ridge-loo} and~\eqref{eq:ridge-change}]
Define the retained normal matrix $H_{-i}=A_{-i}^\top A_{-i}+(n-1)\lambda I=M-a_ia_i^\top$. For $n\ge2$ and $\lambda>0$, both $H_{-i}$ and $M$ are positive definite. Setting $t=a_i^\top H_{-i}^{-1}a_i\ge0$, Sherman--Morrison rank-one update formulas yield adjusted leverage $\widetilde h_i=t/(1+t)<1$, ensuring positive leverage denominators $1-\widetilde h_i>0$.

The stationarity condition for retained ridge fitting is
\[
(M-a_ia_i^\top)w_{-i}=A^\top y-a_i y_i.
\]
Multiplying by $M^{-1}$ and isolating the weight vector gives $w_{-i}=\widetilde w+M^{-1}a_i(a_i^\top w_{-i}-y_i)$. Taking the inner product with $a_i^\top$, subtracting target $y_i$, and solving for residual $r_{-i}=a_i^\top w_{-i}-y_i$ yields the exact leave-one-out residual identity in Equation~\eqref{eq:ridge-loo}.

For prediction change, comparing full minimizer $(M+\lambda I)w_\Dset=A^\top y$ with adjusted fit $M\widetilde w=A^\top y$ gives parameter shift $\widetilde w-w_\Dset=\lambda M^{-1}w_\Dset$. Subtracting predictions yields
\[
a_i^\top(w_{-i}-w_\Dset)
=a_i^\top(\widetilde w-w_\Dset)+(r_{-i}-\widetilde r_i)
=\lambda a_i^\top M^{-1}w_\Dset+
\frac{\widetilde h_i\widetilde r_i}{1-\widetilde h_i},
\]
proving Equation~\eqref{eq:ridge-change}.
\end{proof}

Under an unnormalized penalty $\alpha>0$, full matrix $M=A^\top A+\alpha I$ uses unadjusted full fit $w_\Dset$ and standard leverage $h_i=a_i^\top M^{-1}a_i$, yielding residual ratio $r_{-i}=r_i/(1-h_i)$ and squared loss ratio $r_i^2/\Floor_i=(1-h_i)^2$. When penalty $\lambda$ is defined per sample mean loss, adjusted matrix $M$ provides the exact sample-size normalization correction.

\subsection{Leave-one-out stability: statement and proof}
\label{app:stability}

Leave-one-out stability bounds quantify how record deletion impacts parameter estimation in strongly convex objectives \citep{bousquet2002stability}. Let $n\ge2$, let $\ell_i(w)$ be example losses, and let $r(w)$ be a differentiable regularizer. Objective functions $F_\Dset(w)=\frac1n\sum_{i\in\Dset}\ell_i(w)+r(w)$ and $F_\Rset(w)=\frac1{n-1}\sum_{i\in\Rset}\ell_i(w)+r(w)$ are $\mu$-strongly convex ($\mu>0$) with minimizers $w_\Dset,w_\Rset$. Assuming loss gradients are bounded $\|\nabla\ell_i(w)\|\le G$ and predictions are $L_p$-Lipschitz in parameters, parameter shift satisfies $\|w_\Dset-w_\Rset\|\le2G/(\mu n)$, proving Equation~\eqref{eq:stability}.

\begin{proof}[Proof of the leave-one-out stability bound]
Stationarity of the retained objective gives $\nabla r(w_\Rset)=-(n-1)^{-1}\sum_{i\in\Rset}\nabla\ell_i(w_\Rset)$. Substituting into the full gradient $\nabla F_\Dset(w_\Rset)$ yields
\[
\nabla F_\Dset(w_\Rset)=\frac1n\left(\nabla\ell_f(w_\Rset)
-\frac1{n-1}\sum_{i\in\Rset}\nabla\ell_i(w_\Rset)\right),
\qquad \|\nabla F_\Dset(w_\Rset)\|\le\frac{2G}{n}.
\]
Strong convexity of $F_\Dset$ at stationarity $\nabla F_\Dset(w_\Dset)=0$ implies
\[
\mu\|w_\Rset-w_\Dset\|^2
\le\langle\nabla F_\Dset(w_\Rset),w_\Rset-w_\Dset\rangle
\le\frac{2G}{n}\|w_\Rset-w_\Dset\|.
\]
Dividing by $\|w_\Rset-w_\Dset\|$ yields $\|w_\Rset-w_\Dset\|\le 2G/(\mu n)$. Applying prediction Lipschitz continuity $|f_{w_\Dset}(x_f)-f_{w_\Rset}(x_f)|\le L_p\|w_\Dset-w_\Rset\|$ completes the proof of Equation~\eqref{eq:stability}.
\end{proof}

The $O(1/n)$ asymptotic decay relies on uniform gradient bounds $G$ and positive lower bound $\mu>0$. For mean-squared loss, $\mu=2\lambda$ provides strong convexity, though unrestricted regression lacks global gradient bounds. Thus, algorithmic stability provides a conditional theoretical mechanism for low prediction sensitivity, whereas exact ridge identities provide unconstrained exact calculations.

\section{Experimental details}
\label{app:protocol}

\subsection{Controlled family construction}
To systematically analyze the impact of structural redundancy on deletion floors, we construct a synthetic benchmark in a six-dimensional descriptor space using $300$ cosine Random Fourier Features (RFFs). Projection frequencies are drawn from a standard Gaussian distribution $N(0, I)$, phases are sampled uniformly over $[0, 2\pi)$, and latent target function weights are initialized from a Gaussian distribution with standard deviation $0.30$. Target evaluations share a fixed latent function and a held-out set of $1{,}500$ noiseless test points across all sweep evaluations.

Each generated dataset contains $90$ prototype family centers (coordinate standard deviation $1.5$) with within-family perturbations (standard deviation $0.05$) alongside $160$ isolated background points, yielding labels corrupted by additive Gaussian noise with standard deviation $0.02$. Model training uses ridge regression with regularization parameter $\lambda=10^{-4}$ under mean-squared error. Across family sizes $k\in\{1,2,4,8,16\}$ (scaling total dataset size as $90k+160$), we evaluate $300$ independently generated dataset draws, refitting ridge weights after deleting a family member and an isolated control separately. The resulting family-size sweep is summarized in Figure~\ref{fig:controlled}.

\begin{figure}[htbp]
\centering
\begin{minipage}{.62\linewidth}
\includegraphics[width=\linewidth]{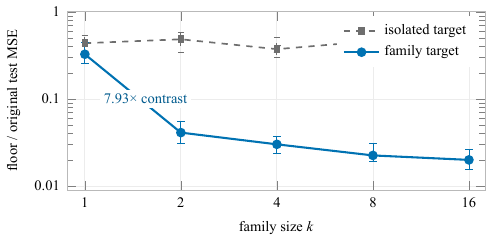}
\end{minipage}
\caption{\textbf{Retained relatives can keep a deleted family member predictable.} Increasing family size from $k=1$ to $k=2$ leaves one relative after deletion and sharply lowers the median normalized retraining loss. Isolated controls remain higher throughout the sweep. Loss is divided by each draw's original-model test MSE, and error bars are $95\%$ bootstrap intervals over $300$ generated datasets and targets per family size. All $90$ families grow together, coupling redundancy with total training size.}
\label{fig:controlled}
\end{figure}

\subsection{Materials representation and reference}
The real-world evaluation uses $44{,}344$ crystal structures extracted from the MP export, filtered by maximum site count $n_{\mathrm{sites}}\le20$, energy above hull $\le0.08$~eV/atom, and formation energy $\le2$~eV/atom. Structures are randomly partitioned into $35{,}476$ training entries and $8{,}868$ test entries, targeting formation energy per atom in eV/atom.

Each crystal structure is parameterized by $134$ features: $94$ elemental composition fractions, $8$ global geometric attributes, and a $32$-bin pair-distance radial histogram spanning $0$--$12$~\AA. Standardized feature vectors $x$ map through a cosine RFF kernel $\phi(x)=\sqrt{2/p}\cos(Wx+b)$, where weights $W_{ij}\sim N(0,\gamma^2)$ and bias $b_j\sim\mathrm{Unif}[0,2\pi)$. Kernel bandwidth $\gamma$ is set to the inverse median non-zero Euclidean distance across $4{,}000$ randomly sampled training pairs. Feature standardization, bandwidth parameters, and RFF projections remain fixed during deletion, ensuring that exact retraining refits linear weights under identical mean-loss penalties. The two fitting regimes and their train/test errors are summarized in Table~\ref{tab:ridge-config}.

\begin{center}\begin{minipage}{\linewidth}\centering\small
\captionof{table}{\textbf{The higher-feature, lower-penalty setting fits both training and test data better.} Both settings use the same train/test partition, and MSE is reported in (eV/atom)$^2$. The table compares fitting regimes with feature count and regularization coupled. Diagnostic requests are sampled separately; the paired comparison uses the same $200$ targets.}
\label{tab:ridge-config}
\begin{tabular}{@{}lrrrr@{}}\toprule
Features $p$ & $\lambda$ & Diagnostic requests & Train MSE & Test MSE\\\midrule
$2{,}048$ & $10^{-6}$ & $200$ & $0.0753784$ & $0.0869666$\\
$8{,}192$ & $10^{-8}$ & $300$ & $0.0166086$ & $0.0703859$\\\bottomrule
\end{tabular}\end{minipage}\end{center}

\subsection{Request sampling and uncertainty}
Diagnostic deletion requests are sampled uniformly without replacement from the training set, reserving a shared cohort of $200$ requests for paired-regime and operating-point evaluations.

The empirical proxy \textsc{FloorScore} identifies the nearest retained structural neighbor in standardized descriptor space. Its regularity estimate $\widehat L$ corresponds to the $99.5$th percentile of prediction difference quotients across $20{,}000$ sampled pairs, while residual estimate $\widehat\varepsilon$ matches the original model's $90$th-percentile absolute training residual. Score tertiles are computed independently within each diagnostic sample. Statistical uncertainty is quantified via $95\%$ percentile bootstrap confidence intervals \citep{efron1979bootstrap}, drawing $10{,}000$ resamples for ridge strata and $2{,}000$ resamples for synthetic, operating-point, and neural diagnostics.

\subsection{Approximate-update settings}
Approximate unlearning procedures initialize from independent copies or refits of the original full-data model, using training data for unlearning updates and reserved held-out data exclusively for test utility evaluation.

\begin{center}\small
\captionof{table}{\textbf{Fixed update settings define the operating points being compared.} Step counts and objectives specify each intervention's strength; all gradient methods use learning rate $10^{-2}$ and retain minibatches of at most $256$ points. The results characterize the selected interventions at these fixed strengths; comparative tuning lies outside this experiment.}
\label{tab:hyper}
\begin{tabular}{@{}lrl@{}}\toprule
Procedure & Steps & Additional setting\\\midrule
Gradient ascent & $30$ & Target squared-loss ascent\\
NegGrad+ \citep{kurmanji2023unbounded} & $30$ & Forget/retain mixing $\alpha=0.5$\\
Retain fine-tuning & $60$ & Retain squared-loss descent\\
SCRUB-style regression & $40$ & 20 ascent, 20 retain; weight $0.3$\\
Newton-style update & $1$ & Hessian damping $10^{-2}$\\
Random relabeling & $40$ & Replacement target drawn from retained labels\\\bottomrule
\end{tabular}\end{center}

Algorithm variants adapt standard unlearning formulations to regression tasks: SCRUB-style regression alternates target loss ascent with retain-task and teacher-prediction matching under squared loss \citep{kurmanji2023unbounded}; random relabeling replaces the target label with a randomly sampled retained label during retain fine-tuning; and Newton-style removal adapts Hessian-based unlearning \citep{guo2020certified} via a damped retained Hessian correction scaled by the forget fraction. Utility impact is reported as median test MSE relative to the original model.

\section{Supplementary results}
\label{app:results}

\subsection{Complete operating-point summaries}
Evaluating approximate unlearning procedures across shared deletion requests isolates differences between target loss shifts, baseline reference agreement, and retained model utility.

\begin{center}\small
\captionof{table}{\textbf{Target-error increases can move an update farther from the retraining reference.} Gradient ascent has the largest median absolute loss gap and test-error ratio among these operating points. A gap near zero indicates similar target losses; a target-loss ratio above one indicates greater error than retraining. Entries summarize $200$ shared requests in $2{,}048$-feature ridge, with brackets giving $95\%$ bootstrap intervals. Gap values are in units of $10^{-3}$ (eV/atom)$^2$, and test ratios use original-model test MSE. The retrain gap and its interval are below $10^{-9}$ in the displayed units.}
\label{tab:all-operating}
\setlength{\tabcolsep}{4pt}
\begin{tabular}{@{}lccc@{}}\toprule
Procedure & Loss gap $\times10^3$ & Target loss / floor & Test ratio\\\midrule
Retrain (reference) & $<10^{-9}$ & $1.000\ [1.000,1.000]$ & $1.000$\\
No update & $1.72\ [1.07,2.93]$ & $0.932\ [0.922,0.939]$ & $1.000$\\
Retain fine-tuning & $1.83\ [1.23,2.86]$ & $0.930\ [0.918,0.938]$ & $1.000$\\
SCRUB-style regression & $3.65\ [2.28,5.38]$ & $1.141\ [1.132,1.156]$ & $1.001$\\
NegGrad+ & $18.85\ [11.58,24.24]$ & $1.645\ [1.631,1.659]$ & $1.013$\\
Gradient ascent & $62.58\ [37.40,78.31]$ & $3.081\ [3.048,3.116]$ & $1.083$\\
Newton-style update & $1.66\ [1.01,2.82]$ & $0.934\ [0.925,0.941]$ & $1.000$\\
Random relabeling & $1.92\ [1.28,2.58]$ & $0.928\ [0.909,0.936]$ & $1.000$\\
\bottomrule\end{tabular}\end{center}

The absolute loss gap $g_i=|e_i-\Floor_i|$ and target loss ratio $q_i=e_i/\Floor_i$ provide complementary perspectives on reference agreement across all eight procedures evaluated on the shared $200$-request cohort (Table~\ref{tab:all-operating}, with hyperparameter configurations specified in Table~\ref{tab:hyper}). The bootstrap intervals overlap for several conservative updates, but this alone establishes neither statistical equivalence nor the absence of paired differences. Distributional retraining equivalence cannot be inferred from these scalar summaries.

\subsection{Exploratory neural descriptor diagnostic}
To examine whether retained-support patterns extend beyond the fixed-feature kernel models, we evaluate single-seed retraining losses from a Crystal Graph Convolutional Neural Network (CGCNN) \citep{xie2018cgcnn} on the same crystal population. The $134$-dimensional descriptors are used only to define the retained-neighbor diagnostics. The CGCNN architecture comprises three message-passing layers of width $64$, processing graphs with up to $12$ neighbors within an $8$~\AA\ cutoff and $41$ Gaussian distance features. Model training proceeds for $60$ epochs using the Adam optimizer \citep{kingma2015adam} (batch size $128$, learning rate $10^{-2}$, zero weight decay), achieving a baseline test MSE of $0.0029175$~(eV/atom)$^2$.

Each of the $200$ deletion requests was evaluated by a separate from-scratch CGCNN retraining run. For request $i$, only record $i$ was removed from the original $35{,}476$-record training split, and the model was trained on the remaining $35{,}475$ records for the full $60$ epochs; the $8{,}868$ held-out test records were excluded from training. Network weights, Adam optimizer state, and the learning-rate scheduler were initialized anew for every run, using the same seed, rather than reusing the original model or a previous request's model. Deletions were evaluated independently, not cumulatively. The reported loss is the squared prediction error at the removed record after that run. These are single-seed realizations of the retraining reference, not averages over training randomness. The request-bootstrap intervals quantify sampling uncertainty across deletion requests, conditional on the fixed training seed; they do not capture variability across training seeds.

For each deleted target, we identify its nearest retained descriptor neighbor under fixed standardization and measure rank associations using Spearman correlation without fitting parameters to deletion losses.

\begin{center}\small
\captionof{table}{\textbf{Neighbor label gaps show a stronger association than descriptor distance.} Larger label gaps are positively associated with higher neural retraining loss. Descriptor distance has little rank association, and its interval includes zero. Entries are Spearman correlations over $200$ requests with $95\%$ request-bootstrap intervals conditional on one training seed. The correlations quantify association within this sample.}
\label{tab:neural-geometry}
\begin{tabular}{@{}lrr@{}}\toprule
Diagnostic & Correlation & $95\%$ interval\\\midrule
Nearest retained descriptor distance & $0.0367$ & $[-0.0913,\,0.1700]$\\
Label gap to that neighbor & $0.2128$ & $[0.0729,\,0.3464]$\\\bottomrule
\end{tabular}\end{center}

Within this single-seed diagnostic, neighbor label disparity shows a stronger positive association with post-deletion loss than raw descriptor distance (Table~\ref{tab:neural-geometry}). Descriptor distance has little rank association ($\rho=0.0367$, with a confidence interval spanning zero), whereas the target label gap shows a positive association ($\rho=0.2128$). This pattern is consistent with local property agreement contributing to retraining predictability in this neural-model setting.

\end{document}